\documentclass{article}

\usepackage{amsthm}

\newtheorem{theorem}{Theorem}
\usepackage{arxiv}
\usepackage{natbib}
\usepackage[utf8]{inputenc} %
\usepackage[T1]{fontenc}    %
\usepackage{hyperref}       %
\usepackage{url}            %
\usepackage{booktabs}       %
\usepackage{amsfonts}       %
\usepackage{nicefrac}       %
\usepackage{microtype}      %
\usepackage{lipsum}
\usepackage{graphicx}
\graphicspath{ {./images/} }

\newtheorem{Theorem*}{Theorem}

\newtheorem{Claim*}[Theorem]{Claim}

\newtheorem{CounterExample*}{$\overline{\hbox{\bf Example}}$}

\newtheorem{Example*}[Theorem]{Example}

\newtheorem{Intuition*}[Theorem]{Intuition}
\newtheorem{Joke*}[Theorem]{Joke}

\newtheorem{Lemma*}[Theorem]{Lemma}
\newtheorem{Open problem}[Theorem]{Open problem}

\newtheorem{Question*}[Theorem]{Question}


\newcounter{parentnumber}

\usepackage{amsmath}

\newcommand{\ignore}[1]{}

\def\ignore#1{}

\usepackage{graphicx}
\usepackage{subcaption}
\usepackage{float}
\usepackage{xcolor}      
\usepackage{booktabs}

\usepackage{algorithm, algpseudocode}
\usepackage{multirow}
\usepackage{mathtools}
\usepackage{url}
\usepackage{hyperref}
\usepackage{soul}
\usepackage{bbm}
\usepackage{amsmath}
\usepackage{amssymb}
\usepackage{amsthm}
\usepackage{adjustbox}
\usepackage{booktabs}
\usepackage{enumitem}
\usepackage{bbm}

\newcounter{algsubstate}
\renewcommand{\thealgsubstate}{\alph{algsubstate}}

\def \score {\textrm{conformity score}}

\def \Atrain {A^{\textrm{train}}}
\def \Aval {A^{\textrm{val}}}

\def \Atest {A^{\textrm{test}}}

\def \Wtrain {W^{\textrm{train}}}
\def \Wval {W^{\textrm{val}}}

\def \Wtest {W^{\textrm{test}}}
\def \Wcalib {W^{\textrm{calib}}}

\def \Etrain {\mathcal{E}^{\textrm{train}}}
\def \Eval {\mathcal{E}^{\textrm{val}}}
\def \Ect {\mathcal{E}^{\textrm{ct}}}
\def \Etest {\mathcal{E}^{\textrm{test}}}
\def \Ecalib {\mathcal{E}^{\textrm{calib}}}

\def \loss {\mathcal{L}}

\title{Enhanced Route Planning with Calibrated Uncertainty Set}

\author{
 Lingxuan Tang \\
  Department of Systems Engineering\\
  City University of Hong Kong\\
  \texttt{lxtang3-c@my.cityu.edu.hk} \\
   \And
 Rui Luo \\
  Department of Systems Engineering\\
  City University of Hong Kong\\
  \texttt{ruiluo@cityu.edu.hk} \\
  \And
 Zhixin Zhou \\
  Alpha Benito Research \\
  \texttt{zzhou@alphabenito.com} \\
  \And
 Nicolo Colombo \\
  Department of Computer Science\\
  Royal Holloway, University of London\\
  \texttt{nicolo.colombo@rhul.ac.uk} \\
}

\begin{document}
\maketitle

\begin{abstract}
This paper investigates the application of probabilistic prediction methodologies in route planning within a road network context. Specifically, we introduce the Conformalized Quantile Regression for Graph Autoencoders (CQR-GAE), which leverages the conformal prediction technique to offer a coverage guarantee, thus improving the reliability and robustness of our predictions. By incorporating uncertainty sets derived from CQR-GAE, we substantially improve the decision-making process in route planning under a robust optimization framework. We demonstrate the effectiveness of our approach by applying the CQR-GAE model to a real-world traffic scenario. The results indicate that our model significantly outperforms baseline methods, offering a promising avenue for advancing intelligent transportation systems.
\end{abstract}

\section{Introduction}
The route planning problem involves determining the most effective paths for a vehicle or a set of vehicles, taking into account various limitations and goals. This issue is crucial in logistics, transportation, and delivery services, where the efficiency of the routes chosen has a direct effect on the economic success and the quality of service provided.

Route planning problems that account for uncertainty are fundamental to intelligent transportation systems and have garnered growing interest. Existing stochastic shortest path models require the exact probability distributions for travel times and often assume independence between these times. Nonetheless, these distributions are frequently unavailable or imprecise due to insufficient data, and the existence of correlations between travel times on various links has been noted \citep{zhang2017robust}. 
Contextual data \citep{patel2023conformal, qi2020data, guo2023data, sadana2023survey} represents some relevant features that correlate with the unknown uncertainty parameter, which is revealed to assist the decision making under uncertainty. The presence of this contextual data enables the creation of a prediction model to forecast the uncertainty, which in turn informs the development of an uncertainty set within the optimization model.

Advances in graph machine learning techniques offer an opportunity to harness this contextual information within the traffic network, enhancing the uncertainty predictions. 
In graph machine learning, the interaction between nodes is captured by edges with associated weights. In a traffic network, where junctions are depicted as nodes and road sections as edges, the edge weights can represent the traffic capacity of each road. 
The prediction of the edge weights is thus vital to modeling the uncertainty within the transportation system. 

Graph Neural Networks (GNNs) have been successfully used in node classification and link prediction tasks.
In this work, we consider their application to edge weight prediction. 
Edge weight prediction has found use in diverse domains such as message volume prediction in online social networks \citep{hou2017deep}, forecasting airport transportation networks \citep{mueller2023link}, and assessing trust in Bitcoin networks \citep{kumar2016edge}. These examples highlight the wide-ranging applicability and importance of edge weight prediction and load forecasting techniques in different network-based systems.

\section{Related Works}

\subsection{Robust Optimization for Networks}
Robust optimization for networks is a rapidly evolving field that aims to address decision-making problems in network settings under uncertainty. Recent advancements in this field have focused on combining prediction algorithms and optimization techniques to solve decision-making problems in the face of uncertainty\citep{sadana2024survey, patel2024conformal}. Sun et al. \citep{sun2024predict} developed the Contextual Predict-then-Optimize (CPO) framework combining probabilistic weather prediction and risk-sensitive value to provide less conservative robust decision-making under uncertainty in network optimization. 

Incorporating uncertainty on the weight of edges in a graph-based network has been explored through random variables. Previous studies have employed Monte Carlo programming \citep{frank1969shortest} and normal distribution \citep{chen2013finding} to model the uncertainty of edge weights and approach the optimal shortest path. Chassein et al. \citep{chassein2019algorithms} proposed a data-driven approach for robust shortest-path problems in a road network. Experiment results on the Chicago traffic dataset demonstrate that ellipsoidal uncertainty sets performed well, and the authors propose using axis-parallel ellipsoids (ignoring correlations) which allows the development of a specialized branch-and-bound algorithm that significantly outperforms a generic solver.

However, the properties of conventional random variables are unable to accurately deliver the characteristics of uncertainty. Recent studies have introduced data-driven robust optimization, which generates uncertainty representations from existing observational data \citep{bertsimas2018data, chassein2019algorithms}. Robust optimization focuses on optimizing under the worst-case scenario disregarding other information about the uncertainty distribution. To address this limitation, Zhang et al. \citep{zhang2017robust} propose the utilization of $\alpha$-reliable parameters to regulate the level of robustness, aiming to achieve a robustness degree of $1-\alpha$.

\subsection{Traffic Prediction}
Many existing studies on traffic forecasting primarily focus on developing deterministic prediction models \citep{bui2022spatial}. Derrow et al. \citep{derrow2021eta} demonstrates successful application of GNNs in travel time predictions in Google Maps. The model represents the road network as a graph, where each route segment is a node and edges exist between consecutive segments or segments connected via an intersection. The message passing mechanism of GNN enables us to handle not only traffic ahead or behind us, but also along adjacent and intersecting roads

Traffic applications, however, often require uncertainty estimates for future scenarios. Zhou et al. \citep{zhou2020variational} incorporate the uncertainty in the node representations. In \citep{xu2023air}, a Bayesian ensemble of GNN models combines posterior distributions of density forecasts for large-scale prediction. \citep{maas2020uncertainty} combine Quantile Regression (QR) and Graph WaveNet to estimate the quantiles of the load distribution.

Traditionally, traffic forecasting is approached as a node-level regression problem \citep{cui2019traffic, jiang2022graph}, i.e. nodes and edges in a graph represent monitoring stations and their connections.
We adopt an edge-centric approach, i.e. we predict traffic flow over road segments through edge regression. 
Interestingly, the strategy aligns with several real-world setups, e.g. the Smart City Blueprint for Hong Kong 2.0, which emphasizes monitoring road segments (edges) rather than intersections (nodes) \citep{office2019smart}.

\subsection{Conformal Prediction}

Conformal prediction (CP) \citep{vovk2005algorithmic}, is a methodology designed to generate prediction regions for variables of interest, thereby enabling the estimation of model uncertainty by substituting point predictions with prediction regions. This methodology has been widely applied in both classification \citep{luo2024trustworthy,luo2024entropy,luo2024weighted} and regression tasks \citep{luo2024conformal, luo2025volume}. Recently, the application of CP has broadened significantly across various domains, including pandemic-driven passenger booking systems \citep{werner2021evaluation}, smartwatch-based detection of coughing and sneezing events \citep{nguyen2018cover}, and model calibration in the scikit-learn library \citep{sweidan2021probabilistic}.

The conventional approach for CP uncertainty estimation presumes the exchangeability of training and testing data. However, by relaxing this assumption, CP can be adapted to diverse real-world scenarios, such as covariate-shifted data \citep{tibshirani2019conformal},games \citep{luo2024game}, time-series forecasting \citep{gibbs2021adaptive, su2024adaptive}, and graph-based applications \citep{zargarbashi23conformal, luo2023anomalous,huang2023uncertainty, luo2024conformalized,wang2025enhancing}. For instance, Tibshirani et al. \citep{tibshirani2019conformal} expand the CP framework to manage situations where the covariate distributions of training and test data differ, while Barber et al. \citep{barber2023conformal} tackle the more complex issue of distribution drift. Similarly, Luo and Colombo \citep{luo2025conformal} apply CP concepts to graph-based models, and Clarkson \citep{clarkson2023distribution} enhances localized CP to create Neighborhood Adaptive Prediction Sets (NAPS), which assign higher weights to calibration nodes that are proximate to the test node. These advancements illustrate the versatility of CP in addressing various forms of data distribution challenges, thereby enhancing its applicability and robustness in practical implementations.

\section{GNN-Based Traffic Prediction}\label{sec: problem}
Let $G=(\mathcal{V}, \mathcal{E})$ be a road network with node set $\mathcal{V}$ and directed edge set $\mathcal{E} \subseteq \mathcal{V} \times \mathcal{V}$, where nodes and edges represent road junctions and road segments, respectively. 
Assume $G$ has $n$ nodes with $p$ node features such as average and maximum travel speed, and the latitude and longitude.
Let $X \in \mathbb{R}^{n\times p}$ be the node feature matrix. 
The binary adjacency matrix of $G$, 
\begin{equation}
A \in \{0, 1\}^{n\times n}, \quad 
A_{e} =
\begin{cases}
1, & \textrm{if } e \in \mathcal{E}; \\
0, & \textrm{otherwise}.
\end{cases}
\end{equation}
encodes the graph structure.

Define the weight matrix as $W \in \mathbb{R}^{n\times n}$, where $W_{e}$ denotes the traffic cost transitioning of edge $e$. 
We split the edge set into three subsets $\mathcal{E} = \Etrain \cup \Eval \cup \Etest$.
We assume we know the weights of the edges in $\Etrain$ and $\Eval$.
The goal is to estimate the unknown weights of the edges in $\Etest$.
We also assume we know the entire graph structure, $A$. 
To mask the validation and test sets, we define \begin{equation}
\Atrain \in \{0, 1\}^{n\times n}, \quad 
\Atrain_{e} =
\begin{cases}
1, & \textrm{if } e \in \Etrain; \\
0, & \textrm{otherwise}.
\end{cases}
\end{equation}
Similarly, we let $\Aval$ and $\Atest$ be defined as $\Atrain$ with $\Etrain$ replaced by $\Eval$ and $\Etest$. 

The weighted adjacency matrix is
\begin{equation}\label{eq: weighted adj train}
\Wtrain =
\begin{cases}
W_{e}, & \textrm{if } e \in \Etrain; \\
\delta_{e}, & \textrm{if } e \in \Eval \cup \Etest; \\
0, & \textrm{otherwise},
\end{cases}
\end{equation}
where $\delta_{e}$ is a positive number assigned to the unknown edge weights, which can be the average of the existing edge weights (training info from general information) or a value embedded with local neighborhood information (training info from the neighborhood). This processing enables the model to know the graph structure, which is very specific to transportation applications, where we know the road network structure but do not know the traffic flow for some of the roads.

\begin{figure}
\centerline{\includegraphics[width=0.8\textwidth,clip=]{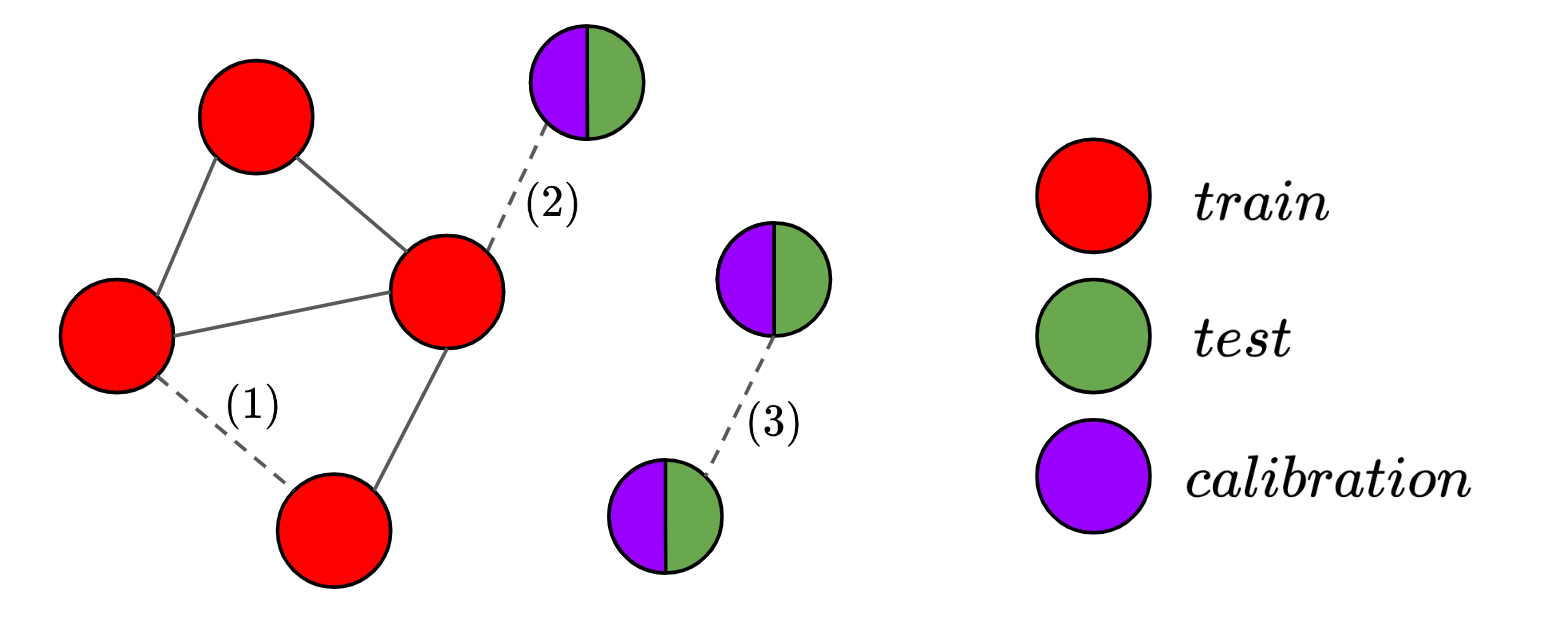}}
\small
\caption{Training settings for edge weight prediction in a conventional data split. 
Different colors indicate the availability of the nodes during training, calibration or testing. 
Solid and dashed lines represent edges used for training and edges within the test and calibration set. 
Predicting (1) corresponds to the transductive setting considered here.
(2) and (3) are examples of the inductive setting.
In road traffic forecasting, (1) may be the undetected traffic flow between two existing road junctions, e.g. for a (new) road where a traffic detector has not yet been installed.
(2) and (3) represent scenarios where new road junctions are constructed, connecting to existing ones or forming connections with each other to create new roads.}
\label{fig: transductive}
\end{figure}

Focusing on road networks, we consider a scenario where we know the entire road systems, $A$, and the traffic volume of certain roads, $\Wtrain \odot \Atrain + \Wval$. $W_{e}$ for $e\in\Etrain\cup\Eval$
The task is to predict the traffic volume of the remaining roads, $\Wtest$. 
During training, the model observes the nodes and leverages their features to make predictions. 
At inference time, the model deduces the edges that connect these nodes.
See Figure \ref{fig: transductive} for a graphical representation of our setup.

To predict the edge weights, we consider using the link-prediction Graph Auto Encoder (GAE).

\subsection{GAE-Based Point Prediction} \label{subsec: GAE}
GAE \citep{kipf2016variational} focuses on learning node embeddings for undirected unweighted graphs.
In addition to undirected unweighted graphs, GAE variants have been extended to graphs with isolated nodes \citep{ahn2021variational}, directed graphs \citep{kollias2022directed}, weighted graphs \citep{zulaika2022lwp}, and graphs with different edge types \citep{samanta2020nevae}. 

To account for the roles of road junctions as either a source or a target in directed road graphs, we adopt a source and a target node embedding matrix, $Z_S$ and $Z_T$, respectively. 
The encoder is a base GNN model\footnote{In the following demonstration, we use the graph convolutional network (GCN) as the base GNN model.}, which learns aggregating information from node neighborhood to update its features:
\begin{equation}\label{eq: DiGAE}
\begin{split}
    H_S^{(l)} &= 
    \begin{cases}
        X, & l=0 \\
        \text{ReLU}\left( {\Wtrain} H_T^{(l-1)} B_T^{(l-1)} \right), & l=1,\dots,L
    \end{cases},
    \\
    H_T^{(l)} &= 
    \begin{cases}
        X, & l=0 \\
        \text{ReLU}\left( {\Wtrain}^\top H_S^{(l)} B_S^{(l)} \right), & l=1,\dots,L
    \end{cases}, \\
    {Z_S} &= {\Wtrain} H_T^{(L)} B_T^{(L)}, \quad  {Z_T} = \langle\Wtrain, H_S^{(L)} B_S^{(L)}\rangle, \\
\end{split}
\end{equation}
where $H_S^{(l)} $ and $H_T^{(l)}$ and $B_S^{(l)}$ and $ B_T^{(l)}$ are the source and target feature and weight matrices at layer $l$.
Note that we use the weighted adjacency matrix $\Wtrain$ (\ref{eq: weighted adj train}) which effectively leverages the entire graph structure.

The decoder produces the predicted weighted adjacency matrix according to
\begin{equation}
    \hat{W} = Z_S {Z_T^\top}.
\end{equation}
To optimize the GNNs parameters, we minimize 
\begin{equation} \label{eq: train DiGAE}
    \loss_{\textrm{GAE}} = \| \Atrain \odot \hat{W} -\Wtrain \|_F.    
\end{equation}
through gradient descent.
We train the model until convergence and then select the parameters that minimize $\loss_{\textrm{GAE}}$ on the validation set, $\Wval$.

In what follows, we describe integrating conformal quantile regression (CQR) \citep{romano2019conformalized} into a GAE model to construct an uncertainty set for edge weights of roads whose traffic volume is unknown.

\subsection{CQR-Based Set Prediction}\label{subsec: CQR}
Since GAE predicts the edge weights based on the embeddings of two adjacent nodes, its outputs are conditionally independent given the node embeddings \citep{jia2020residual}. Conformal quantile regression (CQR) integrates conformal prediction with classical quantile regression,
inheriting the advantages of both approaches \citep{romano2019conformalized}.
We employ CQR to construct edge-weight prediction sets by quantifying the error made by the plug-in prediction interval.

The encoder of the GAE model (Section \ref{subsec: GAE}) is extended by producing a triple output, i.e. three embeddings for each node. 
The decoder utilizes these embeddings for producing the mean, the $\alpha/2$ quantile, and the $(1-\alpha/2)$ quantile of the predicted edge weights. 
This enables parameter sharing which differs from having three single-output GAE encoders.
Let $\hat{W}$, $\hat{W}^{\alpha/2}$, and $\hat{W}^{1 - \alpha/2}$ be the mean, $\alpha/2$, and $(1-\alpha/2)$ quantiles of the edge weights, i.e.
\begin{equation}\label{eq: CQR output}
    f_\theta\left( e; A, X, \Wtrain \right) = \left[ \hat{W}_{e}, \hat{W}^{\alpha/2}_{e}, \hat{W}^{1 - \alpha/2}_{e}  \right].
\end{equation}
We train the embedding by minimizing 
\begin{equation}\label{eq: train CQR-GAE}
    \loss_{\textrm{CQR-GAE}} = \loss_{\textrm{GAE}} + \sum_{e \in \Etrain} \rho_{\alpha/2}(\Wtrain_{e}, \hat{W}^{\alpha/2}_{e}) + \rho_{1 - \alpha/2}(\Wtrain_{e}, \hat{W}^{1-\alpha/2}_{e}), 
\end{equation}
where $\loss_{\textrm{GAE}}$ is the squared error loss defined in (\ref{eq: train DiGAE}) to train the mean estimator, $\hat W$.
The second term is the pinball loss \citep{steinwart2011estimating, romano2019conformalized}, defined as 
\begin{equation}
    \rho_{\alpha}(y, \hat{y}) \coloneqq 
    \begin{cases}
        \alpha (y - \hat{y}) & \textrm{if } y > \hat{y} \\
        (1 - \alpha) (y - \hat{y}) & \textrm{otherwise}
    \end{cases}
\end{equation}
Algorithm \ref{alg: CQR} describes how to obtain the prediction sets, i.e., prediction intervals, given by (\ref{eq: interval CQR}). 

\begin{algorithm}
\caption{Conformal Quantile Regression for Graph Autoencoder}
\label{alg: CQR}
\hspace*{\algorithmicindent} \textbf{Input:} The binary adjacency matrix $A \in \{0, 1\}^{n\times n}$, node features $X\in \mathbb{R}^{n\times p}$, training edges and their weights $\Etrain$, $\Wtrain$, calibration edges and their weights $\Ecalib$, $\Wcalib$, and test edges $\Etest$, user-specified error rate $\alpha \in (0,1)$, GAE model $f_\theta$ with trainable parameter $\theta$.\\
\begin{algorithmic}[1]
\State Train the model $f_\theta$ with $\Wtrain$ according to (\ref{eq: train CQR-GAE}).
\State Compute the \score{} which quantifies the residual of the calibration edge weights $\Wcalib$ projected onto the nearest quantile produced by $f_\theta$ (\ref{eq: CQR output}):
\begin{equation}\label{eq: CQR score}
    V_{e} = \max\left \{ \hat{W}^{\alpha/2}_{e} - W_{e}, W_{e} - \hat{W}^{1-\alpha/2}_{e}  \right\}, \; e \in \Ecalib.
\end{equation}
\State Compute $q =$ the $k$th smallest value in $\{V_{e}\}$, where $k=\lceil(|\Ecalib| +1)(1-\alpha)\rceil$;
\State Construct a prediction interval for test edges:  
\begin{equation}\label{eq: interval CQR}
    C_{e} = \Big[\hat{W}^{\alpha/2}_{e} - q, \hat{W}^{1-\alpha/2}_{e} + q \Big], \; e \in \Etest.
\end{equation}
\end{algorithmic}
\end{algorithm}

We provide prediction intervals for the test edges $e \in \Etest$ with the coverage guarantee:
\begin{equation}\label{eq: coverage guarantee}
    P\big(\Wtest_{e} \in C_{e} \big) \geq 1 - \alpha.
\end{equation}
To demonstrate the coverage guarantee, we adopt the permutation invariance assumption \citep{huang2023uncertainty} for edges. %
As a result of the permutation invariance condition for the GNN training, we arrive at the following conclusion. 

\begin{theorem}
    Given any score $V$ satisfying the permutation invariance condition and any confidence level $\alpha\in (0,1)$, $V_{e}$ is a simple random sample from $\{V_{e}\}_{e\in \{\Ecalib \cup \Etest\}}$. 
    We define the split  conformal prediction set as $\hat{C}_e =\{W_e: V_e(W_e)\le q\}$, where
    $
    q = \inf \{\eta: \sum_{e \in |\Ecalib|} \mathbbm{1}\{V_e(W_e) \le \eta\}\ge(1-\alpha)(1+|\Ecalib|) \}$, $e \in \Ecalib$.
    Then $P(W_e\in\hat{C}_e)\ge 1-\alpha$, $e \in \Etest$.
\end{theorem}

\begin{proof}
    We take into account the complete unordered graph, all attribute and label details, and the index sets $\Etrain$ and $\Ecalib\cup\Etest$. 
    The unordered set of $\{V_k\}^{|\Ecalib|+|\Etest|}_{k=1}$ is fixed as
    $\{V_k\}_{k\in \{\Ecalib\cup\Etest\}}$, and $\{V_k\}_{k=1}^{|\Ecalib|}$ is a simple random sample from $\{V_k\}_{k\in \{\Ecalib\cup\Etest\}}$. 
    As a result, any test sample $V_e(W_e)$, $e\in \Etest$ is exchangeable. 
    Under the standard framework of conformal prediction \citep{vovk2005algorithmic}, this methodology ensures valid marginal coverage, i.e., $P(W_{e} \in \hat{C}_e) \ge 1 - \alpha$, with the expectation taken over all sources of randomness.
    Hence, the coverage guarantee holds for (\ref{eq: coverage guarantee}).
    
\end{proof}

\subsection{ERC-Enhanced Set Prediction}\label{subsec: ERC}
One disadvantage of CQR is that it produces inefficient prediction sets as the correction provided in CQR (Algorithm \ref{alg: CQR}) does not account for data heteroscedasticity \citep{romano2019conformalized} and cannot adapt to data. We generalize Error Re-weighted Conformal approach (ERC) \citep{papadopoulos2008normalized, colombo2023training} to the prediction intervals of CQR-GAE, which reweight the \score{} by a fitted
model of the residuals.

We argue that a benefit of CQR-GAE is that it obtains a measure quantifying residuals for free, without any further training. In particular, we use $|\hat{W}^{1-\alpha/2}_{e} - \hat{W}^{\alpha/2}_{e}|$ to quantify the variability of $e$'s edge weight.

The enhanced prediction intervals are given by
\begin{equation}\label{eq: interval ERC}
    C_{e} = \Big[\hat{W}^{\alpha/2}_{e} - q |\hat{W}^{1-\alpha/2}_{e} - \hat{W}^{\alpha/2}_{e}|, \hat{W}^{1-\alpha/2}_{e} + q |\hat{W}^{1-\alpha/2}_{e} - \hat{W}^{\alpha/2}_{e}| \Big], \; e \in \Etest,
\end{equation}
where $q$ is the $\lceil(|\Ecalib| +1)(1-\alpha)\rceil$-th smallest value in the \score{s} evaluated on the calibration set:
\begin{equation}\label{eq: ERC score}
    V_{e} = \max\left \{ \frac{\hat{W}^{\alpha/2}_{e} - W_{e}}{|\hat{W}^{1-\alpha/2}_{e} - \hat{W}^{\alpha/2}_{e}|}, \frac{W_{e} - \hat{W}^{1-\alpha/2}_{e}}{|\hat{W}^{1-\alpha/2}_{e} - \hat{W}^{\alpha/2}_{e}|}  \right\}, \; e \in \Ecalib.
\end{equation}

\section{Robust Optimization for Route Planning}\label{sec: route planning}
Route planning is a long-standing point of interest in transportation analytical problems \citep{patel2024conformal, mor2022vehicle}. We specifically consider the task of identifying the shortest route from a source node to a target node in a directed graph, defined as the P2P problem \citep{goldberg2005computing}. This problem has significant relevance across various fields, including the calculation of driving directions. 

In this section, we focus on incorporating the prediction intervals obtained from Section \ref{subsec: CQR} into route planning problems to account for decision making under uncertainty.

\subsection{Robust Optimization with Predicted Uncertainty Sets} %
Given the calibrated uncertainty sets for the traffic flow (edge weights) through CQR, we formulate a robust optimization problem that deals with shortest path planning under uncertainty.

Consider the road network $G=(\mathcal{V}, \mathcal{E})$ represented as a weighted directed network graph with $(s, t)$ as the source-target pair. A prediction interval for the traffic flow on each road segment given by the CQR-GAE (Algorithm \ref{alg: CQR}) can be related to a prediction range for the travel cost $c$ on each road \citep{patel2024conformal}. The shortest path problem with edge weight uncertainty aims to find a path from the source to the destination that minimizes the maximum possible travel cost under this uncertainty. Due to the non-negativity of the decision variables $x$, the robust optimization problem can be simplified by considering only the upper bounds of the cost intervals:

\begin{equation}\label{eq: simplified robust shortest path}
\begin{split}
    \min_{x \in [0, 1]^{|\mathcal{E}|} } & \langle c^{\max}, x\rangle \\
    \textrm{s.t. } & Bx=b
\end{split}
\end{equation}
where $c^{\max}\in \mathbb{R}^{|\mathcal{E}|}$ is the edge weight vector composed of the maximum predicted values for each element corresponding to the respective edge. $B \in \{0, 1\}^{|\mathcal{V}|\times|\mathcal{E}|}$ is the node-arc incidence matrix. $b \in \mathbb{R}^{|\mathcal{V}|}$ has source node $b[s] = 1$, target node $b[t] = -1$, and $b[k] = 0$ for $k \notin \{s, t\}$.

However, such a simplification neglects the minimum travel cost $c^{\min}$ for each road, thereby losing a significant amount of useful information that the predictive model offers. To avoid this waste of information, we set the value-at-risk rate to obtain less conservative decisions with the same level of risk guarantee and incorporate side observations into prediction to reduce uncertainty on the objective.

\subsection{Contextual Risk-sensitive Robust Optimization}
Contextual robust optimization considers the optimization problem with the presence of covariates induced by side observations. The classical linear programming problems do not involve covariates. However, due to the presence of uncertainty on travel cost, the objective is highly related to physical occasions, e.g., weather conditions, day of the week \citep{agrawal2019machine, franch2020precipitation}. With the covariates, denoted as $z$, we arrive at the following contextual robust shortest path problem,
\begin{equation}\label{eq: contextual RO shortest path}
\begin{split}
    \min_{x \in [0, 1]^{|\mathcal{E}|} } & \left[\langle c^{\max}, x\rangle|z\right] \\
    \textrm{s.t. } & Bx=b
\end{split}
\end{equation}
where covariates $z$ depict the conditional situation on side observations. In our situation, the node embedding (\ref{eq: DiGAE}) is obtained by the GNN model.

To mitigate the conservatism of this system, a risk-sensitive objective is introduced,
\begin{equation}\label{eq: contextual RO VaR shortest path}
\begin{split}
    \min_{x \in [0, 1]^{|\mathcal{E}|} } & \left[\langle\textrm{VaR}_{\alpha} (c), x\rangle|z\right] \\
    \textrm{s.t. } & Bx=b
\end{split}
\end{equation}
where $\alpha\in (0, 1)$. Here $\textrm{VaR}_{\alpha}(U)$ denotes the $\alpha$-quantile/value-at-risk (VaR) of a random variable $U$; specifically, $\textrm{VaR}_{\alpha}(U):=F^{-1}_{U}(\alpha)$ with $F^{-1}_{U}(\cdot)$ be the inverse cumulative distribution function of $U$. From the information viewpoint, the traditional robust system (\ref{eq: contextual RO shortest path}) primarily focuses on predicting $[c^{\max}|z]$, while solving (\ref{eq: contextual RO VaR shortest path}) may require a distributional prediction of the conditional distribution $[\Tilde{c}|z]$. We circumvent this requirement on an entire distribution by using the quantiles produced by the proposed algorithms.

\section{Experiments}
In this section, we evaluate the prediction intervals produced by the CQR-based method (Section \ref{subsec: CQR}, \ref{subsec: ERC}) on a real traffic network. 
We then perform experiments on robust route planning (Section \ref{sec: route planning}) using two real-world traffic datasets. We evaluate the travel costs arising from various prediction intervals generated by CQR-based methods with those from baseline approaches.

\subsection{Traffic Network Analysis}\label{sec: empirical}
We apply CQR to a real-world traffic network, specifically the road network and traffic flow data from Chicago \citep{bar2021transportation}. The Chicago dataset consists of 546 nodes representing road junctions and 2150 edges representing road segments with directions. In this context, each node is characterized by a two-dimensional feature $X_i\in \mathbb{R}^{2}$ representing its coordinates, while each edge is associated with a weight that signifies the traffic volume passing through the corresponding road segment.

\begin{figure}
\centerline{\includegraphics[width=\textwidth,clip=]{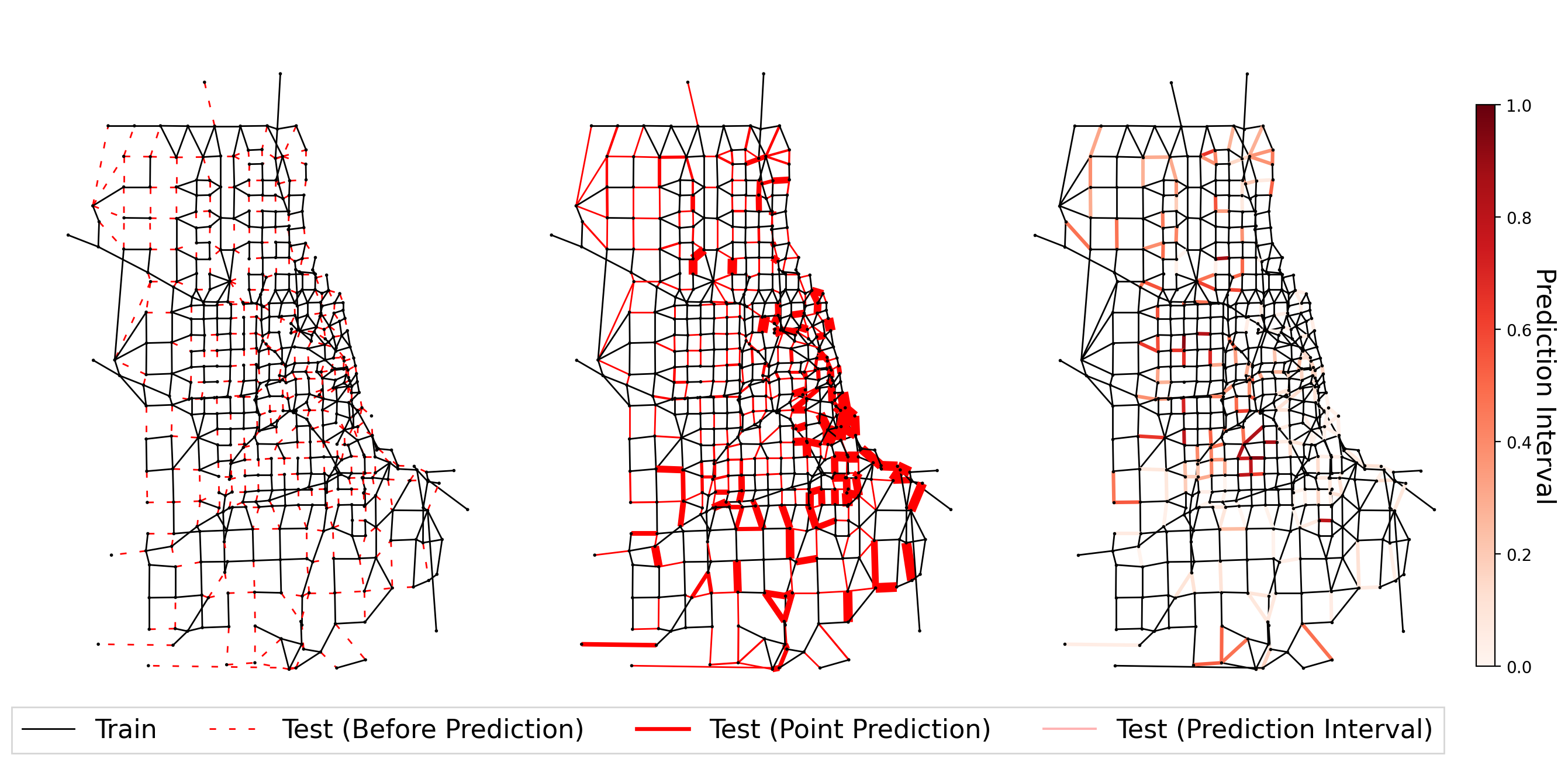}}
\small
\caption{The figure demonstrates the application of our proposed prediction models, which provide a coverage guarantee, using a snapshot of road network and traffic flow data from Chicago, IL, United States \citep{bar2021transportation}. The road network is divided into training roads (represented by black solid lines) and test roads (represented by red dashed lines). Our CQR-GAE model (Algorithm \ref{alg: CQR}) is developed to generate a prediction interval with a user-specified error rate of $\alpha=0.05$. The middle plot displays the predicted edge weights $\hat{W}$, where the line thickness increases proportionally with the predicted edge weights. The right plot illustrates the lengths of the prediction intervals, with darker lines indicating wider intervals or higher inefficiency.}
\label{fig: Chicago}
\end{figure}

We adopt a similar procedure from \citet{jia2020residual, huang2023uncertainty}, and allocate 50\%, 10\%, and 40\% for the training set $\Etrain$, validation set $\Eval$, and the combined calibration and test set $\Ect$, respectively. Figure \ref{fig: Chicago} provides an example of how the Chicago network data is divided into training/val/test/calibration edges. Additionally, the prediction outcome of our proposed CQR-GAE (Algorithm \ref{alg: CQR}) is depicted. The median plot shows the predicted edge weights, while the right-hand plot shows the width of the prediction interval.

\subsection{Baseline}
We consider traditional uncertainty sets as a baseline. This includes a basic Interval-based approach (Basic) that overlooks the graph's structure, and a Quantile Regression (QR) based method that develops the uncertainty sets through interpolative methods using information from adjacent edges.

\begin{enumerate}[label=(\arabic*), leftmargin=20pt]
    \item \textbf{Interval uncertainty set:} 
    Every edge is assigned an identical uncertainty interval $[c^{\textrm{mean}} +\lambda (c^{\textrm{min}} - c^{\textrm{mean}}), c^{\textrm{mean}} + \lambda (c^{\textrm{max}} - c^{\textrm{mean}}) ]$. Here, $c^{\textrm{mean}} = \frac{\sum_{e \in \Etrain} W_{e}}{|\Etrain|} $ represents the average cost, $c^{\textrm{max}} = \max\limits_{e \in \Etrain} W_{e}$, $c^{\textrm{min}} = \min\limits_{e \in \Etrain} W_{e}$ are the maximum and minimum costs, and $\lambda \geq 0$ is a scaling parameter.

    \item \textbf{Quantile regression-based uncertainty set:} 
    For the quantile regression-based method, the prediction sets are constructed using a GNN similar to (\ref{eq: DiGAE}) but does not apply the correction as in CQR, i.e., the prediction intervals are not corrected by the $d$ computed from the \score{s} on the calibration set in Algorithm \ref{alg: CQR}. Compared with the Basic approach, QR exploits the knowledge of graph structure. 
    
\end{enumerate}

In contrast to a recent study \citep{patel2023conformal} that deals with the robust traffic flow problem by devising travel costs based on weather predictions derived from radar network precipitation readings, our work focuses on data imputation within traffic systems. We aim to directly predict the travel costs for roads with incomplete data and establish the corresponding uncertainty set through the use of a graph neural network model.

\subsection{Evaluation Metrics}
We evaluate the performance of baseline methods with our CQR (Section \ref{subsec: CQR}) and CQR-ERC (Section \ref{subsec: ERC}) methods using metrics including the empirical coverage and $\alpha$ worst-case performance. 
For evaluating the route planning performance with different uncertainty sets, we use the true edge weights, defined as 
\begin{equation}\label{eq: true cost}
    \textrm{cost} = \textrm{vec}(W)^\top x^{*},
\end{equation}
where $x^{*} = \arg\min_{x} \langle c^{\max}, x\rangle$.

For evaluating whether various methods produce prediction sets that contain the true edge weights as expected, we use the marginal coverage, defined as 
\begin{equation}\label{eq: cover}
    \textrm{cover} = \frac{1}{|\Etest|} \sum_{e\in \Etest} \mathbbm{1}\big(\Wtest_{e} \in C_{e}\big),
\end{equation}
where $C_{e}$ is prediction interval for edge $e$.

\subsection{Results}

Baseline methods determine uncertainty regions of edge weights using probable values for each edge but overlook correlations based on edge connections within the network. This oversight leads to an inflated uncertainty volume and less informed decision-making. Once again, we provide evidence of the measurable enhancement in decision-making outcomes that arise from utilizing the more informative CQR(Section \ref{subsec: CQR}) and CQR-ERC (Section \ref{subsec: ERC}) prediction intervals. Experiments were conducted where the values of $(s,t)$ were randomly chosen from the set $\mathcal{V}$ following a uniform distribution. The outcomes of these experiments can be found from Table \ref{table: cost} to Table \ref{table: cover}. It is evident that the optimal decisions based on the Basic and QR methods, when applied in practice, result in costs that are inferior to those of the CQR and CQR-ERC methods while the prediction intervals generated by QR cannot achieve coverage guarantees. However, the CQR-ERC method, after reweighting the scores in CQR, exhibits improved performance in both coverage and optimal decision-making.

\begin{table}[H]
    \centering
    \caption{Optima cost was tested adopting Baseline, QR, CQR, CQR-ERC (training info from general information) over a set of 20 i.i.d. test samples and do 100 permutations per sample with 95\%-quantile of cost $\Tilde{c}$ with corresponding standard deviations indicated in parentheses.}
    \label{table: cost}
   \begin{tabular}{ccccc}
   \toprule
      & Baseline & QR & CQR & CQR-ERC \\
   \midrule
    cost & 138.781(20.11) & 137.383(20.20) & 137.312(20.18) & 137.310(20.18) \\
   \bottomrule
\end{tabular}
\end{table}

\begin{table}[h]
    \centering
    \caption{Optima cost was tested adopting Baseline, QR, CQR, CQR-ERC (training info from the neighborhood) over a set of 20 i.i.d. test samples and do 100 permutations per sample with 95\%-quantile of cost $\Tilde{c}$ with corresponding standard deviations indicated in parentheses.}
    \label{table: xxx}
   \begin{tabular}{ccccc}
   \toprule
      & Baseline & QR & CQR & CQR-ERC \\
   \midrule
    cost & 159.388(18.18) & 156.474(16.91) & 156.519(16.86) & 156.517(16
    .85) \\
   \bottomrule
\end{tabular}
\end{table}

\begin{figure}[h]
  \centering
  \includegraphics[width=\textwidth, clip=]{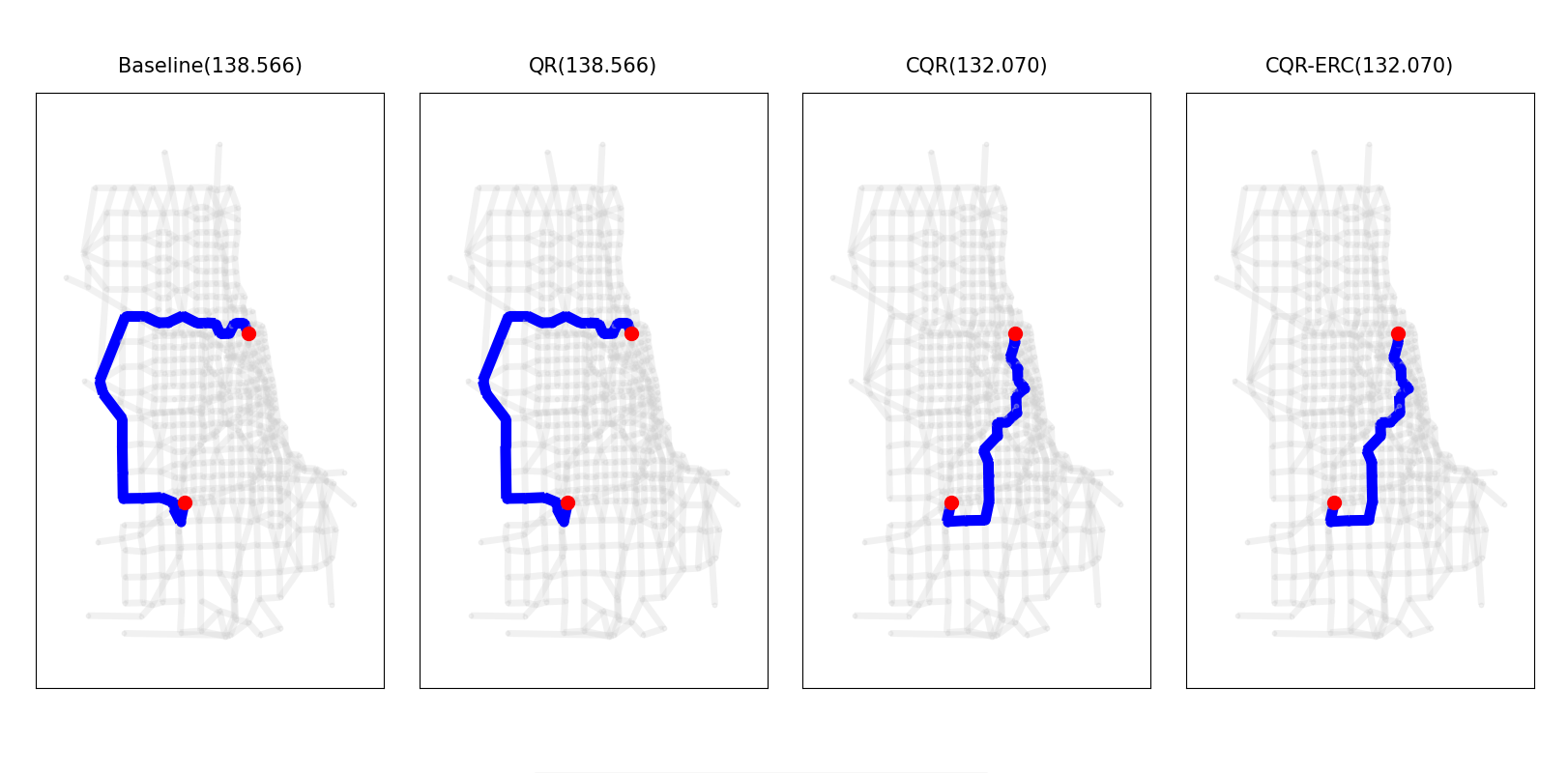} 
  \caption{The figure illustrates various decision paths generated for a given source-target pair under Baseline, QR, CQR, CQR-ERC (training info from general information). Each subplot is accompanied by a descriptive title indicating the algorithm responsible for generating the respective path, along with the actual cost associated with it.} 
  \label{fig:example} 
\end{figure}

\begin{figure}[H]
  \centering
  \includegraphics[width=\textwidth, clip=]{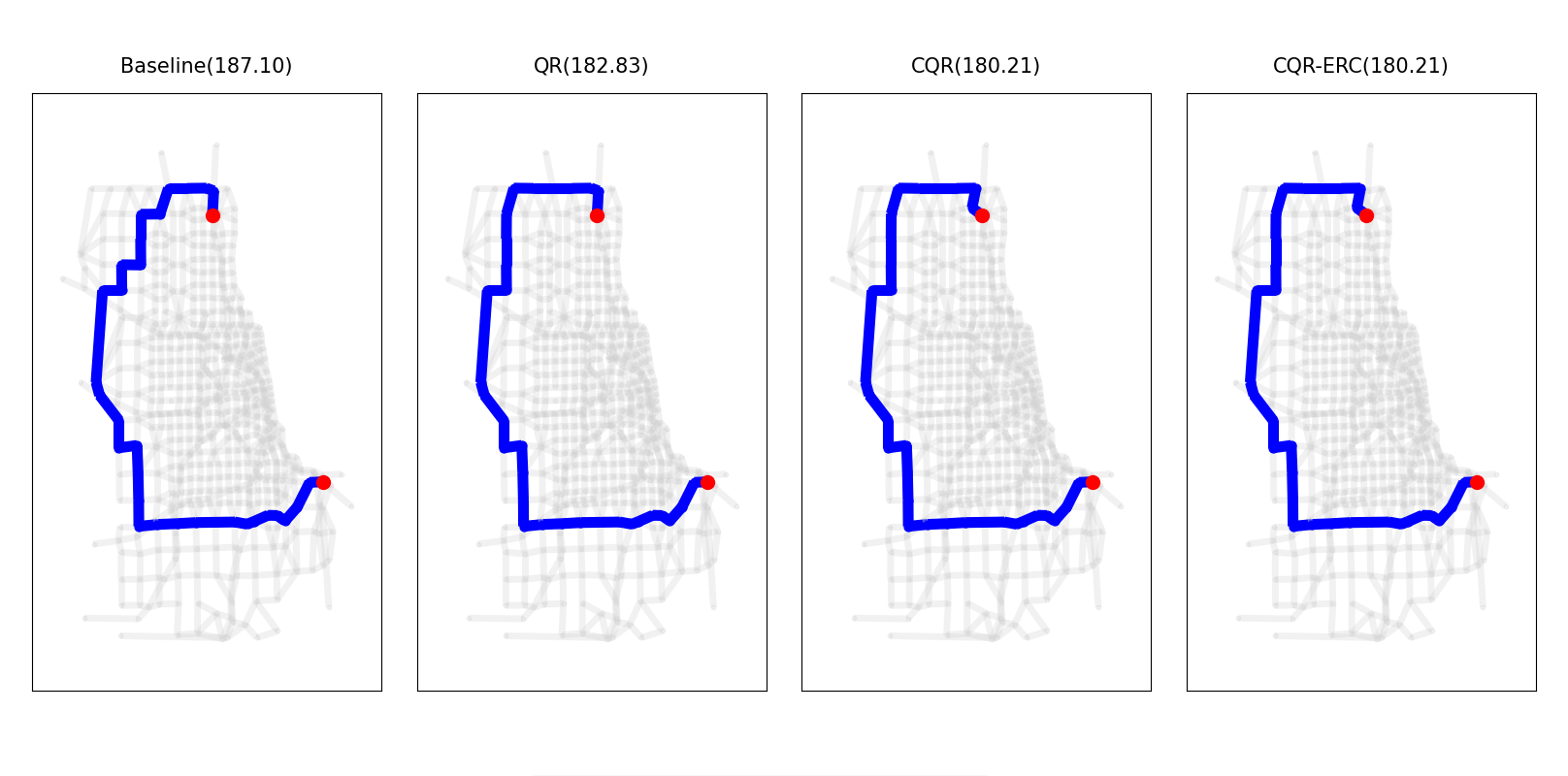} 
  \caption{The figure illustrates various decision paths generated for a given source-target pair under Baseline, QR, CQR, CQR-ERC (training info from the neighborhood). Each subplot is accompanied by a descriptive title indicating the algorithm responsible for generating the respective path, along with the actual cost associated with it.}
  \label{fig:example2} 
\end{figure}

\begin{table}[H]
    \centering
    \caption{The table shows the coverages of prediction intervals generated by the Baseline, QR, CQR, and CQR-ERC methods (training info from general information) for true edge weights under different permutation counts. The Baseline method uses intervals from 0 to the mean, while QR, CQR, and CQR-ERC use intervals at $\alpha=0.05$. Values in parentheses indicate interval sizes. The QR method's coverage does not reach 0.95, whereas CQR and CQR-ERC achieve and maintain coverage around 0.95. With fewer permutations, interval sizes vary significantly, but as permutations increase, interval sizes decrease, and the differences among the methods narrow.}
    \label{table: cover 1}
   \begin{tabular}{ccccc}
   \toprule
    number of permutation & Baseline & QR & CQR & CQR-ERC \\
   \midrule
    50 & 0.5020 & 0.9393(3.77) & 0.9515(3.97) & 0.9514(4.14) \\ 
    100 & 0.4949 & 0.9384(3.5823) & 0.9511(3.75) & 0.9511(3.79) \\ 
    1000 & 0.5014 & 0.9391(3.52) & 0.9512(3.67) & 0.9512(3.68) \\ 
   \bottomrule
\end{tabular}
\end{table}

\begin{table}[H]
    \centering
    \caption{The table shows the coverages of prediction intervals generated by the Baseline, QR, CQR, and CQR-ERC methods (training info from the neighborhood) for true edge weights under different permutation counts. The results also show that CQR and CQR-ERC achieve 0.95 coverage with fewer permutations.}
    \label{table: cover}
   \begin{tabular}{ccccc}
   \toprule
    number of permutation & Baseline & QR & CQR & CQR-ERC \\
   \midrule
    50 & 0.4920 & 0.9446(3.59) & 0.9513(3.66) & 0.9513(3.67) \\ 
    100 & 0.5081 & 0.9421(3.65) & 0.9512(3.73) & 0.9512(3.74) \\ 
    1000 & 0.5001 & 0.9499(3.76) & 0.9513(3.76) & 0.9513(3.76) \\ 
   \bottomrule
\end{tabular}
\end{table}

\begin{figure}[h]
\centering
\subfloat[interval by QR]{
		\includegraphics[scale=0.15]{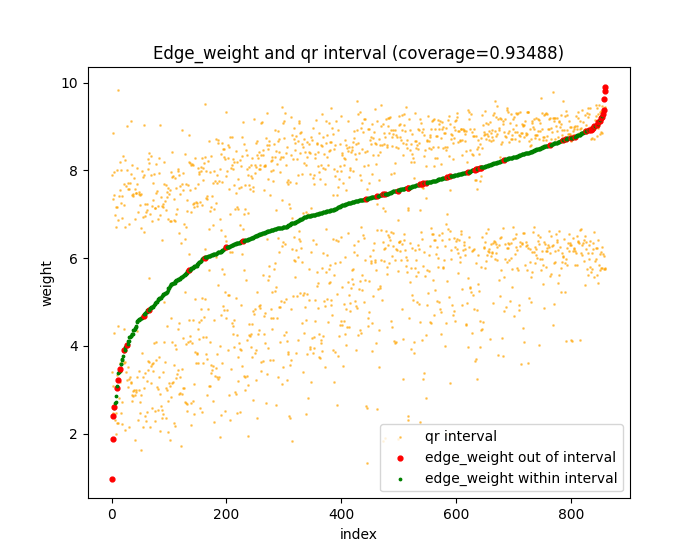}}
\subfloat[interval by CQR]{
		\includegraphics[scale=0.15]{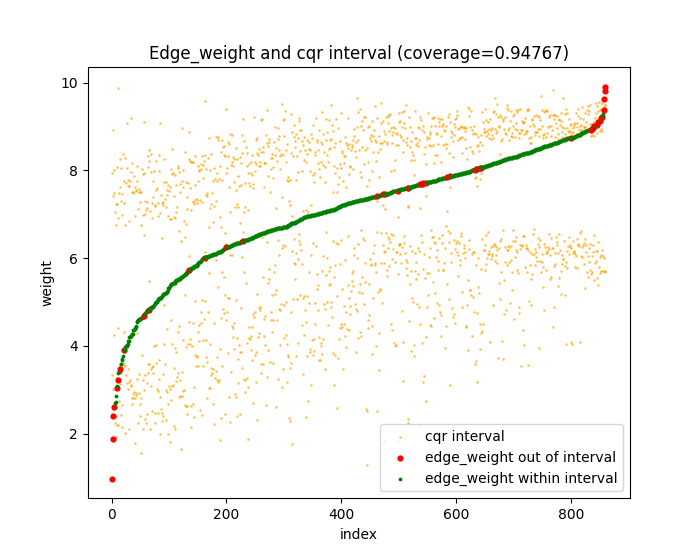}}
\subfloat[interval by CQR-ERC]{
		\includegraphics[scale=0.15]{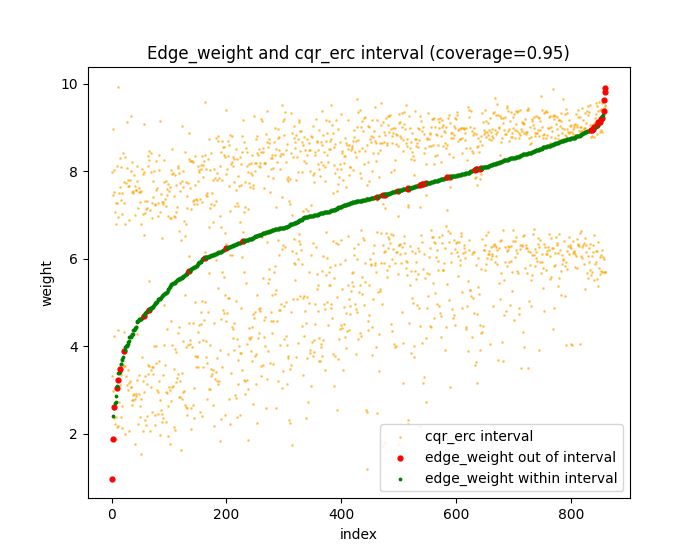}}
\caption{
This figure compares the prediction intervals generated by the QR, CQR, and CQR-ERC methods (training info from general information) at $\alpha=0.05$ with the true edge weights.
}
\label{Fg: intervals}
\end{figure}

\begin{figure}[h]
\centering
\subfloat[interval by QR]{
		\includegraphics[scale=0.15]{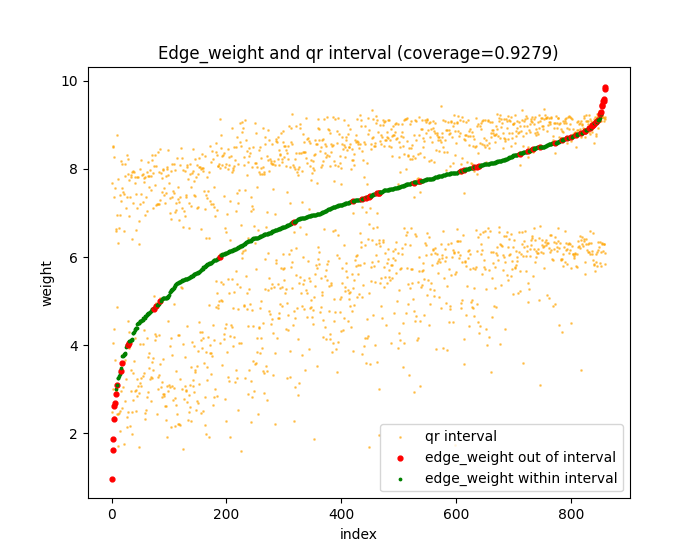}}
\subfloat[interval by CQR]{
		\includegraphics[scale=0.15]{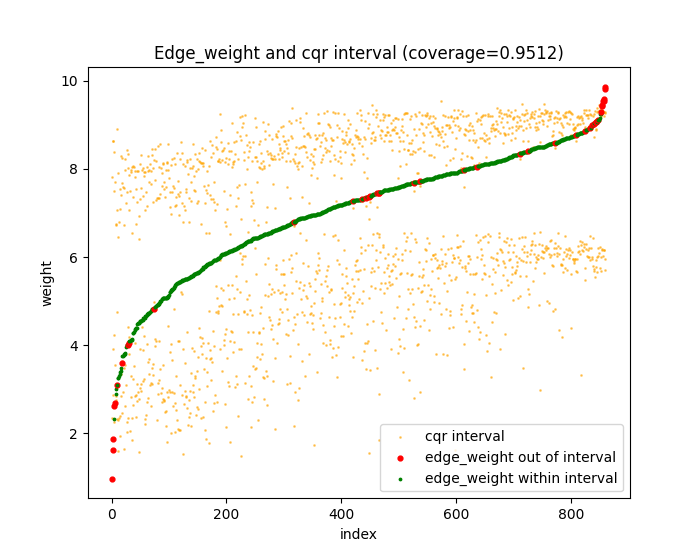}}
\subfloat[interval by CQR-ERC]{
		\includegraphics[scale=0.15]{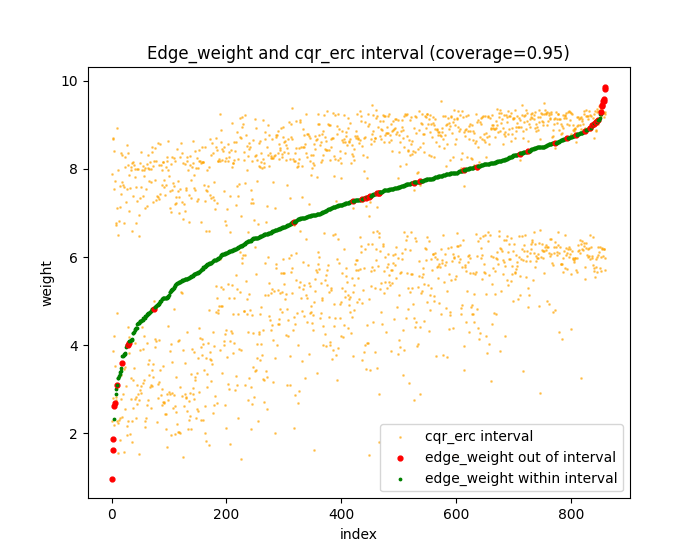}}
\caption{
This figure compares the prediction intervals generated by the QR, CQR, and CQR-ERC methods (training info from the neighborhood) at $\alpha=0.05$ with the true edge weights.
}
\label{Fg: intervals2}
\end{figure}
In Figure \ref{Fg: intervals} and Figure \ref{Fg: intervals2},  green dots indicate edges whose true weights fall within the intervals, while red dots indicate those that do not. The QR method shows significantly worse coverage than the CQR and CQR-ERC methods, as evidenced by more red dots. The CQR method improves coverage, and the CQR-ERC method achieves the highest accuracy, with coverage closest to $1-\alpha$.

\newpage
\section{Conclusion}\label{sec: conclusion}
In this paper, we explore the advantages of probabilistic prediction over point prediction, where the former predicts an entire expected distribution function instead of a single value. We enhance these probabilistic predictions by employing conformal prediction techniques, specifically CQR-GAE, which provide a coverage guarantee. By integrating uncertainty sets generated by a conformal prediction algorithm tailored to graph neural networks, we enhance decision-making in route planning within a robust optimization framework. Our evaluation of the proposed CQR-GAE model in a real-world traffic scenario shows its superior performance compared to baseline methods.

\end{document}